\documentclass{article}

 \usepackage[main, final]{neurips_2026}

\usepackage[utf8]{inputenc} 
\usepackage[T1]{fontenc}    
\usepackage{hyperref}       
\usepackage{url}            
\usepackage{booktabs}       
\usepackage{amsfonts}       
\usepackage{nicefrac}       
\usepackage{microtype}      
\usepackage{xcolor}         
\usepackage{amsmath,amssymb,amsthm}
 \usepackage{graphicx}
\newtheorem{theorem}{Theorem}[section]
\newtheorem{proposition}[theorem]{Proposition}
\newtheorem{corollary}[theorem]{Corollary}
\newtheorem{lemma}[theorem]{Lemma}

\theoremstyle{definition}
\newtheorem{definition}[theorem]{Definition}

\title{How Hard Is It to Rig a Benchmark? A Social Choice Analysis of Leaderboard Robustness}

\author{%
  \textbf{Polina~Gordienko}\textnormal{\textsuperscript{1,2,3,4}}
  \qquad
  \textbf{Georg~Schollmeyer}\textnormal{\textsuperscript{1}} \\[0.5em]
  \textbf{Frauke~Kreuter}\textnormal{\textsuperscript{1,2,3,5}}
  \qquad
  \textbf{Christoph~Jansen}\textnormal{\textsuperscript{6}} \\[1em]
  \textnormal{\textsuperscript{1} Department of Statistics, LMU Munich} \\
  \textnormal{\textsuperscript{2} Munich Center for Machine Learning (MCML)} \\
  \textnormal{\textsuperscript{3} relAI -- Konrad Zuse School of Excellence in Reliable AI} \\
  \textnormal{\textsuperscript{4} University of Oxford} \\
  \textnormal{\textsuperscript{5} Social Data Science Center, University of Maryland} \\
  \textnormal{\textsuperscript{6} School of Computing \& Communications, Lancaster University Leipzig} \\[0.5em]
  \textnormal{\texttt{Polina.Gordienko@stat.uni-muenchen.de}}%
}

\begin{document}

\maketitle

\begin{abstract}
Multi-task benchmarks have become a central pillar of machine learning research, yet their growing influence has incentivised benchmark gaming -- strategic actions taken to improve the leaderboard rank of a specific model. Treating datasets as voters and models as candidates, we consider \textit{benchmark-specific training} -- the inclusion of benchmark data in training -- as a form of election manipulation. For any ordinal benchmark, the problem of choosing datasets to train on so that a target model becomes top-ranked corresponds to shift bribery, a class of manipulation problems from computational social choice. Leveraging this identification, we show that the benchmark-specific training problem is NP-hard under Borda count and mean win rate. Complementing this worst-case perspective, we introduce the \textit{instance-level robustness}, the minimum number of datasets a model developer must include in training to top a given leaderboard, and derive expressions for it under arithmetic mean, median, mean win rate and pairwise majority. We evaluate these expressions on MMLU under HELM and on BIG-Bench Hard (BBH) under the Open LLM Leaderboard. Across both suites, mean win rate is hardest to manipulate: this gap is clear on BBH (24 tasks, 4507 models), where its median robustness is 22 tasks (92\%), compared with 13 (54\%) under arithmetic mean and 12 (50\%) under median and pairwise majority.\footnote{The code for reproducing all experiments is available at \url{https://github.com/polinamgordienko-glitch/How-Hard-Is-It-to-Rig-a-Benchmark-A-Social-Choice-Analysis-of-Leaderboard-Robustness.git}.}
\end{abstract}

\section{Introduction}\label{intro}

Benchmarks are the compass of modern machine learning \citep{hardt2025emerging}. Multi-task suites such as GLUE \citep{wang2018}, SuperGLUE \citep{wang2019}, MMLU \citep{Hendrycks2021}, BIG-Bench \citep{Srivastava2022BeyondTI} and HELM \citep{liang2023} have become the primary mechanism by which the community decides which models are the best. As leaderboard rank correlates with commercial adoption \citep{Chang24}, the incentives surrounding benchmark performance have increased substantially. The result is a mere ``illusion of progress'' \citep{dehghani21}: models that perform well on benchmark tasks but fail on simple challenge examples and falter in real-world scenarios \citep{kiela2021}. 

We use \textit{benchmark gaming} as an umbrella term for strategic actions aimed at improving leaderboard position of a specific model. The most direct form is the inclusion of benchmark data in training, which has been studied as \textit{training on the test set} \citep{DudaHart, hardtrecht2022patterns}, \textit{data contamination} and \textit{benchmark leakage} \citep{sainz-etal-2023-nlp, ni2025surveylargelanguagemodel, Ni25}. Although such practices have been documented among frontier models \citep{singh2025leaderboardillusion} and motivated many detection and decontamination methods \citep{yang2023rethinkingbenchmarkcontaminationlanguage,jiang2024investigatingdatacontaminationpretraining}, their extent in any given suite is hard to verify, since the training data of leading closed-source models is proprietary \citep{zhou2023dontmakellmevaluation}. Instead of asking whether a benchmark is contaminated and how to detect it, we examine the robustness of the benchmark to such manipulation by construction. \textbf{For a fixed suite of datasets and a fixed set of competing models, we ask how many tasks a model developer must include in training to make a target model top-ranked.}

Treating datasets as voters and models as candidates, a multi-task benchmark is a social choice problem and \textit{benchmark-specific training} -- the deliberate inclusion of evaluation datasets in the training of a given model -- is a form of election manipulation. We show that it corresponds to \textit{shift bribery} \citep{Elkind2009, FALISZEWSKI2021}, where an external agent pays a cost to shift a preferred candidate upward in voters' rankings. The question ``how robust is the benchmark?'' then reduces to ``how hard is it to bribe the corresponding election?''. We give two answers. In the worst case, no efficient algorithm can find the cheapest way to bribe the benchmark. For any specific suite, we compute exactly how many datasets a developer must train on to top the leaderboard.

\begin{figure}
\centering
\includegraphics[width=\linewidth]{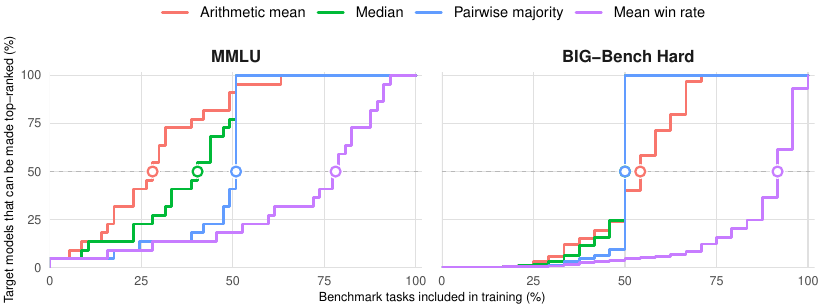}
\caption{How many benchmark datasets must be included in training to top the leaderboard? Treating each model in turn as the target, we compute the minimum fraction of tasks on which benchmark-specific training would make the target top-ranked and plot the empirical CDF across all targets. Curves farther to the right indicate greater robustness. Dots mark the median robustness for each rule.}
\label{ecfd}
\end{figure}

\textbf{Contributions.} We make the following three contributions.

1. We formalize multi-task benchmarking as a preference aggregation problem and benchmark-specific training as a manipulation of it and prove that for any ordinal benchmark operator the resulting problem is exactly shift bribery with \textit{all-or-nothing prices}. As a consequence, benchmark-specific training is NP-hard under Borda count and mean win rate.

2. We introduce \textit{instance-level robustness}, the minimum number of datasets a developer must train on to top a leaderboard, and derive expressions for mean, median, mean win rate and pairwise majority. 

3. We evaluate these expressions on MMLU under HELM (22 models, 57 subjects) and on BIG-Bench Hard under the Open LLM Leaderboard (4507 models, 24 tasks). Mean win rate is consistently the hardest to manipulate, with median robustness of $44.5$ subjects ($78\%$) on MMLU and $22$ tasks ($92\%$) on BBH, against $16$ ($28\%$) and $13$ ($54\%$) under arithmetic mean.

\section{Related Work}\label{lit}

\textbf{Benchmarking and social choice.} The machine learning community has recently turned to social choice theory to address the aggregation problem in benchmarks. The starting point for much of this work has been the critique of mean aggregation. Averaging methods are highly inadequate for benchmarking and may result in misleading leaderboards \citep{ethayarajh20,agarwal21, mishra2021}. \cite{colombo2022bestsystemsnewperspectives, himmi2023robustnlpevaluationhandling, Rofin_2023} investigate the use of voting rules such as Borda, Minimax and Kemeny consensus. \cite{ehl2012, mptbw2015} study benchmarking as a consensus-ranking problem.  \cite{zhang2024inherent} use Arrow's theorem to illustrate a trade-off between diversity and sensitivity in multi-task benchmarks. \cite{gordienko2026arrowimpossibilitypossibilitiesmulticriteria} frame benchmarking as social choice and characterize when aggregation satisfies all Arrow's axioms except for universality.  \cite{lanctot2025evaluatingagentsusingsocial} suggest a framework to evaluate general agents based on concepts from social choice and game theory.

\textbf{Benchmark gaming: contamination, leakage, training on the test task.} A substantial  body of literature studies contamination and leakage, proposing detection and decontamination methods \citep{magar22, roberts2023datacontaminationlenstime, sainz-etal-2023-nlp,yang2023rethinkingbenchmarkcontaminationlanguage,dong-etal-2024, deng-etal-2024-investigating,jiang2024investigatingdatacontaminationpretraining, Ni25}. \cite{kapoor2022leakagereproducibilitycrisismlbased, zhou2023dontmakellmevaluation} document the far-reaching consequences of leakage that go beyond mere overfitting, including reduced adaptation capability, reproducibility failures and unfair advantages in evaluation results. Yet, these practices raise a more fundamental question: what are we justified in concluding about model capabilities from benchmark scores? \cite{freiesleben2025benchmarkingepistemologyconstructvalidity} argue that benchmark results alone measure at best model performance relative to a concrete evaluation dataset and learning problem and that stronger claims require additional validity assumptions. \cite{diddee2026benchbrowsercollectingevidence} show that suites may capture only narrow aspects of a capability and that benchmarks which appear to measure the same capability need not induce the same model rankings. \cite{singh2025leaderboardillusion} give a detailed empirical account of systematic gaming of the Chatbot Arena rating. Distinct from contamination and leakage, \cite{dominguez2025} introduce \textit{training on the test task}, the practice of using knowledge about evaluation tasks during training -- e.g., through instruction-tuning data -- without the training set itself containing test instances.

\textbf{Formal perspectives on gaming.} \cite{blum15} focus on leaderboards in machine learning competitions and propose \textit{the Ladder} mechanism for reliable evaluation. \cite{dwork15} develop a framework that enables adaptive validation while avoiding overfitting to the holdout set. \cite{hardtmegiddo15} deal with strategic classification, where individuals can modify their features in response to a published classifier in order to obtain a more favorable outcome. \cite{chen2026leaderboardincentivesmodelrankings} model benchmarking as a Stackelberg game between a benchmark designer who determines an evaluation protocol and competing model developers who can increase their models' scores by investing in benchmark-specific improvements. They study whether the resulting game has an equilibrium and whether the ranking reflects latent capability. In this paper, we take the complementary perspective: given a fixed score matrix, we ask how robust a given aggregation rule is to manipulation by a model developer.

\section{Formalizing Benchmark Manipulation}\label{framework}
\subsection{Manipulation in Social Choice}
Social choice theory studies how individual inputs -- preferences, judgements or probabilistic beliefs -- are aggregated into a collective output through an \textit{aggregation rule} \citep{sep-social-choice}. A central concern is whether such a rule can be manipulated by a voter misreporting their preferences or by an external agent who \textit{bribes} the voters to change their votes. A famous result of \cite{gibbard} and \cite{ SATTERTHWAITE1975187} shows that once there are at least three alternatives, manipulation is unavoidable for broad classes of aggregation rules, except in degenerate cases when the outcome is determined by a single voter (a \textit{dictator}). Although the Gibbard-Satterthwaite theorem implies that no reasonable aggregation rule is immune to manipulation, it says nothing about how \textit{difficult} manipulation is to carry out. \cite{Bartholdi89} observe that the real threat of manipulation for a given aggregation rule arises only if it is computationally easy to find the optimal action that enables the desirable outcome. A prominent subfield of computational social choice has since investigated the complexity of manipulating different aggregation rules under different assumptions \citep{BARTHOLDI199227,Faliszewski2006, Faliszewski2009, Faliszewski17, FALISZEWSKI2021,Faliszewski_Rothe_2016,Elkind2009, Elkind2020}. 

\subsection{Benchmark-Specific Training}\label{BSTform}

We formalize multi-task benchmarking as a social choice problem where models are alternatives/candidates and tasks are voters. We show that \textit{benchmark-specific training} -- the deliberate inclusion of benchmark datasets in the training process -- is a manipulation of this aggregation.

\textbf{Preference aggregation problem.} Let $\mathcal{D}$ be some fixed and finite universe of datasets/tasks with $m:=|\mathcal{D}|\in \mathbb{N}$ and let $\mathcal{A}$ be some fixed and finite set of models with $n:=|\mathcal{A}|\in \mathbb{N}$.  Let $\phi$ denote an evaluation metric with $\phi: \mathcal{A}\times \mathcal{D} \to [0,1]$. We assume that higher values of metric $\phi$ mean better performance on that metric. Let pref$(\mathcal{A})$ be the set of all complete and transitive binary relations, i.e., \textit{preference relations} on $\mathcal{A}$. Each fixed dataset $D \in \mathcal{D}$ induces a preference relation $\succeq_{D} \in$ pref$(\mathcal{A})$, defined by setting $A \succeq_{D} A^{'} :\Leftrightarrow \phi(A,D) \geq \phi(A^{'},D).$ We denote by $\succ_{D}$ its strict part. Collecting one preference relation for each dataset yields the \textit{profile} $R:=(\succeq_{D})_{D \in \mathcal{D}}.$ A \textit{benchmark operator} is a mapping $B: [0,1]^{n \times m} \to \text{pref}(\mathcal{A}).$ For a fixed $\mathcal{A}$ and a fixed $\mathcal{D}$, $\phi$ determines an $n\times m$ score matrix recording each model's metric value on each dataset; the operator $B$ takes this matrix as input and returns an aggregate preference ordering over all models in $\mathcal{A}$. We write $B(\phi)$ for $B(\phi(A,D)_{A \in \mathcal{A}, D \in \mathcal{D}})$, the leaderboard the benchmark operator $B$ produces when scores are assigned by $\phi$.  We call $B$ \textit{ordinal} if there exists a mapping  $\tilde{B}: \text{pref}(\mathcal{A})^{m} \to \text{pref}(\mathcal{A})$ such that for every $\phi$, $B(\phi)=\tilde{B}(R),$ where $R$ is the profile induced by $\phi$. An ordinal operator uses only the rankings induced by datasets (i.e., which model beats which on each dataset), not the magnitude of differences in values of $\phi$. An operator that is not ordinal is called \textit{cardinal}.

\textbf{Benchmark-specific training.} We begin by introducing a \textit{default evaluation protocol} $P_0$ -- the benchmark's official evaluation procedure, which is fixed independently of the training of any individual model and applied identically to every model in the benchmark. In practice, $P_0$ can be the evaluation harness, e.g., the Eleuther AI Evaluation Harness \citep{eval-harness} used by the Open LLM Leaderboard. Evaluating a model $A\in \mathcal{A}$ on $D \in \mathcal{D}$ under $P_0$ yields $\phi^0 (A,D) \in  [0,1]$. We write $R^0:= (\succeq^{0}_D)_{D \in \mathcal{D}}$ for the resulting profile. Let $A_{1} \in \mathcal{A}$ be a target model under development. The model developer chooses a subset $S\subseteq \mathcal{D}$ to include in the training of $A_{1}$. Since tasks differ in size, complexity and the compute required to train on them, we assign each $D \in \mathcal{D}$ a fixed manipulation cost $c_D>0$. Training on $S$ induces a new metric function: for each $S\subseteq \mathcal{D}$, let  $\phi^T: \mathcal{A}\times \mathcal{D} \to [0,1]$
denote values of $\phi$ assigned to every $A \in \mathcal{A}$ evaluated on every $D \in \mathcal{D}$ such that
\begin{itemize}
    \item  only the target model's performance is affected: $\phi^{0}(A, D)=\phi^{T} (A,D) \ \forall A\neq A_1, \forall D \in \mathcal{D}$;
    \item performance of the target model cannot be decreased: $\phi^{T}(A_1, D)\geq \phi^{0} (A_1,D) \ \forall D \in S$;
    \item performance on remaining datasets is unchanged: $\phi^{0}(A_1, D)=\phi^{T} (A_1,D) \ \forall D \notin S$.
\end{itemize}  

The model developer controls exactly two things: \textit{which} datasets to train on and \textit{how hard} to push the target model's performance on each chosen dataset. Costs may differ across datasets, but the cost of a given dataset does not depend on how far the target model is shifted in that dataset ranking. For each dataset $D \in \mathcal{D}$, let $g_D:=\phi^T(A_1,D) - \phi^0 (A_1,D)$ denote the realized improvement (gain) in metric $\phi$'s value with $g_D=0$ for $D \notin  S$ and $0\leq g_D \leq G_D \ \text{for all} \ D \in  S,$ where  $G_D \in  [0,1-\phi^0(A_1,D)]$  is the maximal gain attainable on $D$. For ordinal operators, we assume \textit{monotonicity}: moving $A_1$ upward in any dataset ranking does not lower its position in the aggregate ranking. For such operators, it is always optimal to achieve maximal gain on every chosen dataset: increasing $g_D$ can only improve the target model's ranking on $D$ and hence, by monotonicity, its aggregate position. This is why we can assume $g_D=G_D$ for all $D \in S$ without loss of generality. Each dataset $D \in \mathcal{D}$ induces the post-training preference relation $\succeq^{T}_D \in$ pref$(\mathcal{A})$. We write $R^T:= (\succeq^{T}_D)_{D \in \mathcal{D}}$ for the resulting profile. We further assume $G_D$ is large enough that the ranking $\succeq^{T}_D$ under $g_D=G_D$ makes $A_1$ top-ranked on $D$. \footnote{This assumption allows the problem in Definition~\ref{BST} to correspond exactly to the shift bribery with all-or-nothing prices \citep{BREDERECK2016140}. Without it, training on $D$ pushes $A_1$ only as far as $G_D$ allows. The identification with shift bribery, Corollary~\ref{borda} and robustness expressions for arithmetic mean, median and mean win rate continue to hold without it; the expressions for pairwise majority in Section~\ref{maj} require it.} If second place on a dataset suffices, first place suffices as well, at the same cost, so the minimum number of datasets is unchanged.

\begin{definition}\label{BST}
The \textbf{Benchmark-Specific Training Problem (BST)} asks, given $\mathcal{A}$, $\mathcal{D}$, a target model $A_1 \in \mathcal{A}$, a default metric $\phi^0$, maximal gains $(G_D)_{D \in \mathcal{D}}$, an operator $B$, costs $(c_D)_{D \in \mathcal{D}}$ and a budget $\beta \in \mathbb{N}$ whether there exists $S\subseteq \mathcal{D}$ with  $\sum_{D \in S}c_D \leq \beta$ such that $A_{1}$ is a top element of $B(\phi^T)$.
\end{definition}

\subsection{Bribing the Benchmark}\label{bribe}

To analyze the computational complexity of Definition~\ref{BST}, we show that it corresponds to a class of \textit{bribery} problems -- a type of election manipulation problem \citep{Faliszewski2006, Faliszewski2009}.

\textbf{Elections and bribery.} An election consists of a finite set of candidates $C$ and a finite set of voters $V$, where each voter has a preference ordering over $C$. An election rule $E$ aggregates the profile of preference relations into a collective ranking of $C$. The top-ranked candidate in the collective ranking is the winner of the election. In the classical bribery problem \citep{Faliszewski2006, Faliszewski2009}, an external agent has a preferred candidate $p \in C$ and a budget $\beta$; she may pay voters to change their preferences and asks whether $p$ can be made the winner under $E$ while spending at most $\beta$. There are different variants of the bribery problem, each of those imposing different restrictions on how voters may be bribed and on how the budget $\beta$ looks like. In \textit{shift bribery} \citep{Elkind2009, BREDERECK2016140, FALISZEWSKI2021}, the agent can only shift $p$ upward in a bribed voter's ranking, while the relative ordering of the remaining candidates remains unchanged. We use the \textit{all-or-nothing} pricing variant of shift bribery \citep{BREDERECK2016140}[Sec. 3.3]: for each voter the agent either pays nothing and the voter's preferences stay untouched or she pays a fixed cost $c$ which can be different for each voter. Under any monotone election rule $E$, it is optimal to shift $p$ to the top of the bribed voter's ranking, since the cost $c$ remains the same regardless of the shift amount.

\textbf{From benchmark-specific training to shift bribery.} The BST problem is structurally identical to shift bribery with all-or-nothing prices. Each task is a voter, each model is a candidate. The model $A_1$ is the preferred candidate $p$ and training $A_1$ on a dataset $D$ corresponds to bribing voter $D$ at cost $c_D$. Benchmark-specific training moves $A_1$ upward in the rankings of those tasks that were included in training, while leaving all other models in the same relative order -- which is exactly the effect of a bribed voter shifting the preferred candidate up. We furthermore restrict our attention to ordinal benchmark operators $\tilde{B}$, where it is always optimal to push $A_1$ to the top of each chosen task's ranking. Hence, the price for each task included in benchmark-specific training is either zero (task not chosen) or a fixed $c_D$ (task chosen, $A_1$ shifted to the top), which is precisely the all-or-nothing pricing variant of shift bribery. We also apply a fixed deterministic tie-breaking rule and, as a consequence, assume that each task induces a strict linear order $\succ_{D}$ on $\mathcal{A}$ for the remainder of this section.

\begin{proposition}\label{iso}
For any monotone ordinal benchmark operator $\tilde{B}$, the BST problem is exactly the shift bribery problem with all-or-nothing prices under the identification of datasets $\mathcal{D}$ with voters $V$, models $\mathcal{A}$ with candidates $C$, the model $A_1$ with the preferred candidate $p$ and the operator $\tilde{B}$ with the election rule $E$.\footnote{The proof of this and the following formal results can be found in Appendix~\ref{proofs}.} 
\end{proposition}

Now we focus on the ordinal operator $B_\text{Borda}$ \citep{borda1781}: on each $D \in \mathcal{D}$, models receive points according to their position in the ranking $\succ_{D}$: $n-1$ points for the top-ranked model, down to zero points for the lowest-ranked model. A model's Borda score is the sum of these points over all datasets. The operator $B_\text{Borda}$ ranks models in decreasing order of Borda score. Although Borda count is used directly in some leaderboards, most notably MTEB \citep{Chung2025MaintainingMT}, it is especially relevant here because under strict preference relations $\succ_{D}$, it induces the same overall ranking as mean win rate, a widely used rule in benchmarking practice \citep{liang2023,hardt2025emerging}.

\begin{corollary}\label{borda}
The BST problem is NP-hard under $B_\text{Borda}$.
\end{corollary}

To prove Corollary~\ref{borda}, we construct a BST instance from a shift bribery instance in polynomial time. The answer to the constructed BST problem is yes if and only if the answer to the given shift bribery problem is yes, with the same costs and budget (Appendix~\ref{proofs}).

It follows that the problem of choosing a subset of benchmark tasks to include in training to guarantee the top rank of the target model is NP-hard under mean win rate. However, NP-hardness merely provides worst-case complexity guarantees, since it merely reflects how hard the problem is in principle \citep{brelsfold}. At the same time, the identification means that existing algorithms for shift bribery with all-or-nothing prices apply to BST under Borda count and, under strict preference relations, mean win rate. These include polynomial-time approximations and an exact fixed-parameter algorithm in the total number of datasets (Appendix~\ref{proofs}). On a given benchmark instance, the model developer may still easily find effective strategies for implementing benchmark-specific training. This motivates the analysis of robustness of different operators to benchmark-specific training on a given instance.

\section{Robustness to Benchmark-Specific Training}\label{robust}
\subsection{Instance-Level Robustness of a Benchmark Operator}
For a specific benchmark instance, we ask how many datasets must minimally be included in the training to guarantee a top rank of  $A_1$. This minimum, which we call the \textit{instance-level robustness}, provides an exact numerical measure of a benchmark's resistance to targeted manipulation for a given operator $B$. A large robustness value means that many datasets must be trained on to ensure the top placement of $A_1$ which raises the cost of training and the probability of detection \citep{dong-etal-2024, jiang2024investigatingdatacontaminationpretraining}. As we measure robustness as the minimum number of datasets that must be included in training to make $A_1$ top-ranked, we adopt the uniform cost $c_D=1$ for each $D \in S$, a special case of the framework in Section~\ref{framework}. Note that NP-hardness does not rely on different costs, since Borda shift bribery is NP-hard under uniform all-or-nothing prices \citep[Thm.~28]{FALISZEWSKI2021}. Under the identification in Proposition~\ref{iso}, the same conclusion holds for BST under Borda and, under strict preference relations, mean win rate.

While the general formulation captures the variability in manipulation cost -- datasets differ in size, difficulty and compute required for training -- instance-level robustness is a \textit{structural} property of the benchmark itself, defined independently of any costs that a developer may face. It is also the quantity directly relevant to the real-life benchmarking practice, as detection methods for contamination focus on \textit{which} datasets appear in training data \citep{yang2023rethinkingbenchmarkcontaminationlanguage}. Measuring robustness by the number of bribed voters is standard in the literature. \cite{Elkind2010} study the optimal amount of voters that need to be targeted in an electoral campaign. This is related to the margin of victory, the minimum number of voters that need to change their votes to alter the election outcome \citep{Dey15}. \cite{Shiryaev,Boehmer2021, Boehmer22} build on this perspective to study the robustness of election winners.

\begin{definition}\label{robustdef}
 Fix $\mathcal D$, $\mathcal{A}$, a target model $A_1 \in \mathcal{A}$, an evaluation metric $\phi: \mathcal{A}\times \mathcal{D} \to [0,1]$ and a default metric $\phi^0$. For a benchmark operator $B$, the \textbf{instance-level robustness} $k_B(\phi^0, A_1)$ is defined as $k_B(\phi^0, A_1):= \text{min} \{|S| : S \subseteq \mathcal{D}, A_1 \ \text{is a top element of} \ B(\phi^T) \ \text{under training on} \ S  \}, $ with $k_B(\phi^0, A_1):= + \infty$ if no such $S$ exists.\footnote{Note that Proposition~\ref{iso} applies only to ordinal benchmark operators. Arithmetic mean and median are cardinal. Sections~\ref{avg} and ~\ref{med} analyze instance-level robustness of these operators and do not rely on the bribery correspondence.}
\end{definition}

\subsection{Arithmetic Mean Aggregation}\label{avg}
The arithmetic mean is the default aggregation rule across many prominent multi-task suites such as GLUE, SuperGLUE and MMLU. It has been extensively criticized for being sensitive to outliers and for combining task scores whose improvements may not be directly comparable \citep{ agarwal21}, yet it remains prevalent in benchmarking practice.

The arithmetic mean is a benchmark operator $B_\text{mean}$ which ranks all models $A \in \mathcal{A}$ according to $\mu_\phi(A)= \frac{1}{m} \sum_{D\in \mathcal{D}} \phi (A, D)$. Define the \textit{mean deficit} that $A_1$ has to overcome to become the top element of $B_\text{mean}(\phi^T)$ by $\Delta_\text{mean}(\phi^0, A_1):= \text{max} \{0, \text{max}_{A \in \mathcal{A}} (\mu_{\phi^0}(A) - \mu_{\phi^0}(A_1)) \}.$ It measures by how much the default performance of $A_1$ falls short of the strongest competing model in the leaderboard. If the target model already holds the highest mean value under default evaluation protocol, this deficit is zero and no benchmark-specific training is needed. Let $G_{(1)}\geq \dots \geq G_{(m)}$ denote the maximal gains for each dataset $D$ sorted in decreasing order. Since benchmark-specific training cannot decrease performance of $A_1$, the sum $\sum_{i=1}^k G_{(i)}$ is the largest total increase in metric score that the model developer can achieve by training the target model on $k$ datasets.

\begin{theorem}\label{mean}
$k_\text{mean}(\phi^0, A_1)= \text{min} \{k \in \{0, \dots, m \}: \sum_{i=1}^k G_{(i)} \geq m \cdot \Delta_\text{mean} (\phi^0, A_1)\}$ with $k_\text{mean}(\phi^0, A_1)=+ \infty$ if no such $k$ exists.
\end{theorem} 

The optimal strategy is simple: sort the datasets by their maximal gain $G_D$ and train on the largest ones until their cumulative gain covers $m$ times the mean deficit. The robustness is small whenever $A_1$ is either already close to the top or has a handful of datasets with large room for improvement and is infinite exactly when training on every available dataset cannot close the gap in mean. 

\subsection{Median Aggregation}\label{med}
Median is the natural robust alternative to mean aggregation, as it is insensitive to how large individual scores are and depends only on how many values lie above and below the middle of the score distribution. Still, it remains a pointwise aggregation rule and has been criticized for leaving out too much information \citep{peyrard-etal-2021-better} and for failing to distinguish between models when many of them share the same median \citep{gera-etal-2025-justrank}. Under median aggregation, $A_1$ becomes top-ranked if enough of its dataset scores reach the highest median among all competing models.

For each $A \in \mathcal{A}$, let $\phi_{(1)}(A) \leq \cdots \leq \phi_{(m)}(A)$ be the order statistics of $\{\phi(A,D): D\in\mathcal{D}\}$ and set $h:= \lfloor \frac{m}{2} \rfloor +1$. The benchmark operator $B_\text{median}$ ranks all models $A \in \mathcal{A}$ by the value of their (upper) median across datasets $\tilde{\phi} (A):= \phi_{(h)}(A)$. To characterize when $A_1$ becomes top-ranked, we introduce the threshold $\tau (\phi^0):= \text{max}_{A \in \mathcal{A} \setminus \{A_1\}} \ \tilde{\phi^{0}} (A)$, which is the highest median attained by any competing model under the default evaluation protocol. Since benchmark-specific training affects only the target model $A_1$, all other models in $\mathcal{A}$ retain their default metric values $\phi^{0}(A, D)=\phi^{T} (A,D)$. Hence, $A_1$ becomes the top element of $B_\text{median}(\phi^T)$ if and only if $\tilde{\phi}^T (A_1) \geq \tau (\phi^0)$, i.e., if the post-training median of the target model is at least as high as $\tau(\phi^0)$. This is equivalent to saying that at least $m-h+1$ datasets satisfy $\phi^T(A_1, D) \geq \tau (\phi^0)$. 

We therefore separate all datasets in the suite into those on which the target model already reaches the threshold and those on which it could be pushed across $\tau (\phi^0)$ by benchmark-specific training. Let $N_{\tau(\phi^0)}(\phi^0, A_1):=|\{D \in \mathcal{D} : \phi^0 (A_1, D) \geq \tau(\phi^0) \}|,$ be the number of datasets on which $A_1$ already reaches $\tau(\phi^0)$ and let $C_{\tau(\phi^0)}(\phi^0, A_1):=|\{D \in \mathcal{D} : \phi^0 (A_1, D) < \tau(\phi^0) \leq \phi^0 (A_1, D) + G_D \}|$ be the number of datasets on which $A_1$ is initially below the threshold but can be pushed to at least $\tau(\phi^0)$ by benchmark-specific training. We define the \textit{median deficit} of $A_1$ by $\Delta_{\text{median}}(\phi^0, A_1):=\text{max} \{0, m-h+1-N_{\tau(\phi^0)}(\phi^0, A_1) \}.$ This is the number of additional datasets on which $A_1$ must reach $\tau(\phi^0)$ in order for its post-training median to cross the threshold. If $A_1$ already has at least $m-h+1$ datasets on which its performance lies above the threshold $\tau(\phi^0)$, then this deficit is zero and no training is needed.
\begin{theorem}\label{median}
$k_\text{median}(\phi^0, A_1)=\Delta_{\text{median}}(\phi^0, A_1) \ \text{if} \ C_{\tau(\phi^0)}(\phi^0, A_1)\geq \Delta_{\text{median}}(\phi^0, A_1).$ If $C_{\tau(\phi^0)}(\phi^0, A_1)<  \Delta_{\text{median}}(\phi^0, A_1)$, then $k_\text{median}(\phi^0, A_1)=+\infty$.
\end{theorem}
While under mean every task with $G_D>0$ contributes to closing the deficit, under median, only tasks that can be pushed from below $\tau(\phi^0)$ to at least $\tau(\phi^0)$ contribute at all. Robustness of the median is either exactly $\Delta_{\text{median}}(\phi^0, A_1)$ when enough datasets are available or $\infty$ when this is not the case.

\subsection{Mean Win Rate}\label{win}
Mean win rate depends only on task rankings: for each dataset, it asks how high a model ranks relative to other models and then averages those win rates across all datasets. This makes it particularly appealing for multi-task benchmarking, since the influence of metric scale and outliers on the overall ranking is removed \citep{hardt2025emerging}.

For each dataset $D \in \mathcal{D}$, we define the win rate of model $A$ by $w_D(\phi,A)=\frac{1}{n} \sum_{A' \in \mathcal{A}} \mathbf{1} \{\phi (A,D) \geq \phi (A',D) \}.$ The operator $B_{win}$ ranks all models $A \in \mathcal{A}$ according to the mean win rate $w (\phi,A)=\frac{1}{m} \sum_{D \in \mathcal{D}}w_D(\phi,A).$ Since $B_{\text{win}}$ is a monotone ordinal operator, we can set $g_D=G_D$ for all $D \in S$ without loss of generality. The pairwise gain of dataset $D$ against $A$ denotes by how much the lead of $A$ over $A_1$ decreases, measured in win rate, as the result of the benchmark-specific training of $A_1$ on $D$. For each $D \in \mathcal{D}$, we consider the benchmark-specific training of $A_1$ on $D$ alone $(S=\{D\})$ and for each $A \neq A_1$ we define:

$$q_D(A):=(w_D(\phi^T, A_1)-w_D(\phi^0, A_1))+(w_D(\phi^0, A)-w_D(\phi^T, A)).$$ 

The first term captures the gain in win rate of $A_1$ on $D$, while the second term captures the loss in win rate of the competing model $A$. The metric value of $A$ remains unchanged, but its win rate can decrease because $A_1$ moves above $A$. Note that $q_D(A)\geq 0$ for all $D\in \mathcal{D}$ and all $A \neq A_1$, since the benchmark specific training can only improve $A_1$'s position relative to $A$. For each $A \neq A_1$, we define the \textit{pairwise win rate deficit} by $d_\text{win}(\phi^0, A_1,A):=w(\phi^0,A) - w (\phi^0,A_1)$ and set $\Delta_{\text{win}}(\phi^0, A_1,A):=\text{max} \{0,d_\text{win}(\phi^0, A_1,A)\}$. Then, we can show that $A_1$ is ranked at least as high as $A$ under $B_\text{win}(\phi^T)$ if and only if $\sum_{D \in S}q_D(A)\geq m \cdot d_\text{win}(\phi^0, A_1,A)$ (see Appendix~\ref{app}). Now $k_\text{win}(\phi^0, A_1,A)$ can be understood as the minimum number $|S|$ such that this inequality holds for a fixed model $A$. For each $A \neq A_1$, let $q_{(1)}(A)\geq \dots \geq q_{(m)}(A)$ denote the pairwise gains sorted in decreasing order.

\begin{theorem}\label{pairwisewin}
$k_\text{win}(\phi^0, A_1,A)=\text{min}\{k \in \{0, \dots, m \}: \sum_{i=1}^k q_{(i)}(A) \geq m \cdot \Delta_{\text{win}}(\phi^0, A_1,A)\}$ with $k_\text{win}(\phi^0, A_1, A)=+ \infty$ if no such $k$ exists.
\end{theorem}
 
Theorem~\ref{pairwisewin} provides a pairwise robustness characterization of mean win rate for a fixed competing model $A$. As in the mean case, the optimal strategy is to train $A_1$ on the datasets with the largest pairwise gains $q_D(A)$ until their sum covers $m$ times the deficit. Making $A_1$ globally top-ranked under $B_{win}(\phi^T)$ requires $\sum_{D \in S}q_D(A)\geq m \cdot \Delta_\text{win}(\phi^0, A_1,A)$ to hold simultaneously for all $n-1$ competing models with the same choice of subset $S$. For each $D \in \mathcal{D}$, set $x_D := \mathbf{1}\{D \in S\}$. Then, we have $\sum_{D\in\mathcal{D}}x_D=|S|$ and obtain the following integer linear program.

\begin{theorem}\label{globalwin}
 $$k_\text{win}(\phi^0, A_1)=\text{min}\sum_{D\in\mathcal{D}}x_D \ \ \text{s.t.} \ \ \sum_{D \in \mathcal{D}}q_D(A) x_D\geq m \cdot \Delta_\text{win}(\phi^0, A_1,A) \ \forall A \in \mathcal{A}\setminus \{ A_1\}$$  with $k_\text{win}(\phi^0, A_1)=+ \infty$ if the program is infeasible.
\end{theorem}

This is a binary covering program: the variables are datasets, the constraints are competing models and $q_D(A)$ is the amount by which choosing dataset $D$ helps close the deficit against $A$. The solution is NP-hard in general: when $q_D(A)\in\{0,1\}$ and
$m\Delta_{\text{win}}(\phi^0,A_1,A)=1$ for every competitor $A$, this becomes exactly set cover \citep{Karp1972}. For the sizes occurring in benchmarking practice, we can compute $k_{\text{win}}(\phi^0,A_1)$ exactly with integer-programming solvers (see Section~\ref{exp}).

\begin{proposition}\label{lowerboundwin}
$k_{\text{win}}(\phi^0, A_1) \geq \text{max}_{A \in \mathcal{A}\setminus \{A_1\}} k_{\text{win}}(\phi^0, A_1, A)$.  
\end{proposition}

Since $A_1$ being top-ranked overall implies, in particular, that it ranks at least as high as each competing model individually, Proposition~\ref{lowerboundwin} provides a lower bound for the robustness of mean win rate.

\paragraph{Comparing aggregation rules.} The pairwise characterization also gives a sufficient condition for comparing win rate with other aggregation rules. Let $B\neq B_{\text{win}}$ have finite robustness $k_B(\phi^0,A_1)$. If, for some competing model $A\in\mathcal A\setminus\{A_1\}$,
\[
\sum_{i=1}^{k_B(\phi^0,A_1)}q_{(i)}(A)
<
m\cdot\Delta_{\text{win}}(\phi^0,A_1,A),
\]
then even the $k_B(\phi^0,A_1)$ largest pairwise gains cannot close the deficit in win rate against $A$. Therefore, it follows from Theorem~\ref{pairwisewin} and Proposition~\ref{lowerboundwin} that $k_{\text{win}}(\phi^0,A_1)>k_B(\phi^0,A_1)$.

\subsection{Pairwise Majority Count}\label{maj}
Mean win rate aggregates pairwise comparisons on each task into a single score for each model and then averages these scores across tasks. Aggregating across tasks first instead yields the pairwise majority count. For each $A,A' \in \mathcal{A}$, $M(\phi, A, A')=\sum_{D \in \mathcal{D}}\mathbf{1} \{\phi (A,D) \geq \phi (A',D) \}$ is the number of tasks on which $A$ ranks at least as high as $A'$. Set $\mu:=\lceil \frac{m}{2} \rceil$ and define the \textit{weak pairwise majority relation} $\succeq_M$ by $A\succeq_M A' \Leftrightarrow M(\phi,A,A')\geq\mu$. Unlike mean win rate, $\succeq_M$ on any pair $A,A'$ depends only on the pairwise comparisons $\{\mathbf{1}\{\phi(A,D)\geq\phi(A',D)\}\}_{D\in\mathcal{D}}$ and is therefore invariant to the addition or removal of other models from the leaderboard.\footnote{Instability to changes of model set implies violation of the pairwise independence of irrelevant alternatives \citep{arrow}.} We call $A^*$ a \textit{weak Condorcet winner} of $\phi$ if $A^*\succeq_M A$ for every $A\neq A^*$. We define $k_{\text{maj}}(\phi^0,A_1)$ as the minimum $|S|$ such that $A_1$ is a weak Condorcet winner of $\phi^T$.\footnote{This is stronger than merely requiring $A_1$ to attain the highest count $\sum_{A'\neq A_1}\mathbf{1}\{A_1 \succeq_M A'\}$. When the weak Condorcet winner is unique, this criterion unambiguously identifies the top model and thus corresponds to Definition~\ref{robustdef}.} For each $A\neq A_1$, let $L^0(A):=\{D\in\mathcal{D}:\phi^0(A,D)>\phi^0(A_1,D)\}$ be the set of tasks on which $A$ strictly beats $A_1$ under the default protocol. Then, it holds that $M(\phi^0,A_1,A)=m-|L^0(A)|$ and we define the \textit{pairwise majority deficit} $\Delta_{\text{maj}}(\phi^0,A_1,A):=\text{max} \{0, \mu-(m-|L^0(A)|)\}$. For a fixed $A$, let $k_{\text{maj}}(\phi^0,A_1,A)$ be the minimum $|S|$ such that $M(\phi^T,A_1,A)\geq \mu$. Furthermore, we use the assumption from Section~\ref{BSTform} on maximal gains: training on any chosen task makes $A_1$ top-ranked on that task.

\begin{theorem}\label{pairwisemaj}
$k_{\text{maj}}(\phi^0,A_1,A)=\Delta_{\text{maj}}(\phi^0,A_1,A)$.    
\end{theorem}

Like Theorem~\ref{pairwisewin}, this characterizes pairwise robustness for a fixed competing model $A$, which is bounded by $\mu=\lceil \frac{m}{2} \rceil$. Making $A_1$ a weak Condorcet winner under $\phi^T$ requires $A_1\succeq_M^T A$ to hold simultaneously for all $n-1$ competing models. Using the variables $x_D$ from Section~\ref{win}, this yields the following integer linear program.

\begin{theorem}\label{globalmaj}
$$k_{\text{maj}}(\phi^0,A_1)=\text{min}\sum_{D\in\mathcal{D}}x_D \ \ \text{s.t.} \ \sum_{D \in L^0(A)}x_D\geq \Delta_\text{maj}(\phi^0, A_1,A) \ \forall A \in \mathcal{A}\setminus \{A_1\}$$    
\end{theorem}

Unlike the robustness of mean, median and mean win rate, $k_{\text{maj}}(\phi^0, A_1)$ is always finite: training on every task satisfies all $n-1$ constraints simultaneously. Solving this program is NP-hard in general, although in practice we can compute it exactly using integer-programming solvers.

\begin{proposition}\label{lowerboundmaj}
$k_{\text{maj}}(\phi^0, A_1) \geq \text{max}_{A \in \mathcal{A}\setminus \{A_1\}} \Delta_\text{maj}(\phi^0, A_1, A)$.  
\end{proposition}

As for mean win rate, equality holds when the same selected datasets that overcome a model with the largest pairwise majority deficit also satisfy all other competitors' constraints. Otherwise, datasets that are useful against one competing model may not cover another and the inequality is strict.

\section{Experiments}\label{exp}
\textbf{Experimental setup.} We implement the framework of Section~\ref{framework} on  MMLU \citep{Hendrycks2021} as evaluated by HELM \citep{liang2023} and BIG-Bench Hard (BBH) \citep{Srivastava2022BeyondTI, suzgun2022challengingbigbenchtaskschainofthought} as evaluated by the Hugging Face Open LLM Leaderboard. For each suite, we treat every individual subject or task as a distinct dataset $D \in \mathcal{D}$ and the benchmark's model set as $\mathcal{A}$. We compute $k_B(\phi^0, A_1)$ for arithmetic mean, median, mean win rate and pairwise majority for each target model $A_1 \in \mathcal{A}$. The two suites differ in evaluation protocol and model coverage, which lets us assess whether the patterns we observe actually reflect properties of the aggregation rule. 

For MMLU, we use the public HELM run logs at version v1.0.0, retaining the 22 of 23 language models with complete scores across 57 subject datasets. Each MMLU subject is treated as a dataset $D \in \mathcal{D}$ and the metric value $\phi(A,D)$ is the score of the HELM accuracy metric \textit{exact\_match} for each model $A$ and each subject $D$. 

For BBH, we use the results from the Open LLM Leaderboard on Hugging Face, which evaluates each task with the Eleuther AI Evaluation Harness \citep{eval-harness}. The leaderboard includes 4507 models, each assessed on 24 reasoning tasks. Each BBH task is treated as a dataset $D \in \mathcal{D}$ and the score $\phi(A,D)$ is \textit{acc\_norm}, the length-normalized accuracy. 

In both cases, the evaluation protocol is fixed across models and serves as the default protocol $P_0$. We retain models with scores on every task and average duplicate entries (for full preprocessing see Appendix~\ref{appexp}). This yields a complete $m \times n$ matrix with tasks as rows and models as columns. For each target model $A_1$ (i.e., each column of the constructed matrix), we compute $k_{\text{mean}}, k_{\text{median}}$, $k_{\text{win}}$ and $k_{\text{maj}}$. When a task is included in training, the target model's post-training score on that task is set to $1$. Additionally, we compute normalized robustness values, as the raw values $k_B(\phi^0, A_1)$ account not only for how vulnerable the aggregation rule is but also for how much room the model has left to improve. A strong model may have a smaller deficit and therefore be easier to bring to the top of the ranking, while a weaker model may have a large deficit but also large available gains $G_D$. The raw robustness values mix these two effects.  We therefore divide $k_B(\phi^0, A_1)$ by a reference value reflecting how much improvement is potentially possible to the target model under that benchmark operator. For pairwise majority, we additionally verify the existence of a Condorcet winner, a model $A^*$ with $M(\phi^0, A^*, A) > M(\phi^0, A, A^*)$ for every $A\neq A^*$ and find that one exists on both suites.

Conditional on the score matrix, the model set and the aggregation rule, each value $k_B(\phi^0,A_1)$ is deterministic. Uncertainty arises only when we summarize these values across target models, e.g., when estimating the fraction of models that can be made top-ranked after training on at most $K$ tasks, interpreting the observed target models as a sample from a broader population of models that could appear on the benchmark. In both suites,  multiple models share a developer or namespace: on MMLU, we have distinct generations of the same model and on BBH, many closely related variants with the same Hugging Face namespace. Treating every entry as independent would overstate our sample size. We therefore report 95\% confidence intervals from a bootstrap that resamples developer namespaces with replacement, recomputing the summary on each resample. For BBH, we additionally recompute all robustness values on the subset that retains only the model with the highest mean task score within each namespace. To compare aggregation rules, we use paired Wilcoxon signed-rank tests and apply Holm correction across rule pairs \citep{Wilcoxon, holm}.

\textbf{Results.}
Across target models on MMLU, the median number of tasks required to top the leaderboard is $16$ subjects ($28\%$) under mean, $23$ ($40\%$) under median, $29$ ($51\%$) under pairwise majority and $44.5$ ($78\%$) under mean win rate. Thus, for a typical target one has to manipulate almost four fifths of the benchmark under mean win rate, compared with just over one quarter under mean. The fraction of models that can be made top-ranked by training on at most five subjects is $14\%$ ($95\%$ bootstrap CI $[0\%, 32\%]$) under mean, $9\%$ $[0\%, 18\%]$ under median and $5\%$ $[0\%, 12\%]$ under both pairwise majority and mean win rate. The confidence intervals are relatively wide, since HELM MMLU contains only a small number of developer namespaces. The paired Wilcoxon signed-rank tests indicate that the differences in robustness of aggregation rules are systematic within the 22 target models: all six rule pairs differ after Holm correction for multiple comparisons ($p_{\text{adj}}<0.001$; Appendix \ref{appexp}).

On BBH, the median robustness across target models is $13$ ($54\%$) tasks under mean, $12$ ($50\%$) under median as well as pairwise majority and $22$ ($92\%$) under mean win rate. Thus, mean win rate requires the broadest manipulation: a typical target must be improved on almost the entire BBH suite to become top-ranked. The normalized robustness analysis in Appendix \ref{appexp} yields the same conclusion: mean win rate has the largest median normalized robustness in both suites. The fraction of models that can be made top-ranked using at most five tasks is $0.98\%$ ($95\%$ bootstrap CI $[0.56\%,1.50\%]$) under mean, $1.04\%$ $[0.57\%,1.65\%]$ under median, $0.33\%$ $[0.12\%,0.60\%]$ under pairwise majority and $0.49\%$ $[0.23\%,0.81\%]$ under mean win rate. As BBH contains many uploads from the same namespace, we repeat the analysis after keeping only the model with the highest mean task score within each of the $718$ namespaces. The median robustness values remain largely unchanged (see Appendix \ref{appexp}). All paired rule comparisons are statistically significant for both analyses ($p_{\text{adj}}<0.001$).

\section{Discussion}\label{disc}
Aggregation rules determine how difficult benchmark manipulation is, so robustness to benchmark-specific training should be treated as a central criterion in benchmark design. Arithmetic mean is among the least robust rules in our framework: concentrated gains on relatively few tasks can be enough to change the leaderboard. Mean win rate requires broad manipulation across the suite, since the target model has to overtake competitors in pairwise comparisons across many tasks. Since manipulation risk is most consequential for models that are already competitive, this diagnostic is especially important near the top of the leaderboard. This robustness comes with a trade-off. Mean win rate depends on the set of models being compared \citep{zhang2024inherent}, while pairwise majority avoids this dependence for any fixed pair but can induce cyclic rankings \citep{condorcet1785}. In both suites we study, this issue does not arise at the top of the leaderboard, since a Condorcet winner exists. HELM initially adopted mean win rate as aggregation rule because its scenarios are evaluated with different metrics and averaging metrics with different scales is problematic \citep{liang2023helmlite}. In 2025, HELM Capabilities used mean score rather than mean win rate, motivated by the dependence of mean win rate on the set of models being compared and by its sensitivity to small variations in scenario scores \citep{xu2025helmcapabilities}. Our results add another criterion to this choice: mean win rate has the largest median robustness in both suites studied in this paper.

A limitation of our analysis is that the manipulation model is deliberately favorable to the developer: training affects only the target model, never lowers its scores outside the chosen subset $S$ and in our tests sets post-training scores to $1$. In practice, leaked data may improve performance only partially and degrade performance on other tasks \citep{zhou2023dontmakellmevaluation}. Extending the framework to include a cardinal budget for the size of metric gains is a natural direction for future work. The main conclusion remains that even under strong assumptions for the developer, aggregation rules differ sharply in how much of the benchmark must be manipulated.

\begin{ack}
We dedicate this paper to the memory of Thomas Augustin. We are deeply grateful for his kindness, mentorship, and encouragement, which have left a lasting impact on us.

\textbf{Funding.}
Polina Gordienko acknowledges support from the Studienstiftung des deutschen Volkes (German Academic Scholarship Foundation) as well as the DAAD programme Konrad Zuse Schools of Excellence in Artificial Intelligence, sponsored by the Federal Ministry of Research, Technology and Space. Frauke Kreuter acknowledges support from the Deutsche Forschungsgemeinschaft (DFG, German Research Foundation) under the National Research Data Infrastructure -- NFDI 27/2 -- 460037581 -- BERD@NFDI. This research was supported by the Federal Ministry of Research, Technology and Space (BMFTR) under grant agreement  16DKZ1029C (KODAQS) and co-funded by the European Union -- NextGenerationEU. The views and opinions expressed are solely those of the authors and do not necessarily reflect the views of the European Union or the European Commission.

\textbf{Competing interests.}
The authors declare no competing interests.
\end{ack}

\bibliographystyle{plainnat}
\bibliography{bibliography}
\newpage
\appendix

\section{Appendix}\label{app}
\subsection{Proofs}\label{proofs}
\begin{proposition}\label{isoapp}
For any monotone ordinal benchmark operator $\tilde{B}$, the BST problem is exactly the shift bribery problem with all-or-nothing prices under the identification of datasets $\mathcal{D}$ with voters $V$, models $\mathcal{A}$ with candidates $C$, the model $A_1$ with the preferred candidate $p$ and the operator $\tilde{B}$ with the election rule $E$.
\end{proposition}
\begin{proof}
We identify datasets $\mathcal{D}$ with voters $V$, models $\mathcal{A}$ with candidates $C$, the target model $A_1$ with the preferred candidate $p$ and the ordinal benchmark operator $\tilde{B}$ with the election rule $E$. For each dataset $D \in \mathcal{D}$, define the all-or-nothing price function by $$\pi_D(0)=0, \ \ \pi_D(l)=c_D \ \ \forall \ l \geq 1,$$
which describes the price of shifting $A_1$ forward in the ranking induced by dataset $D$ by a given number of positions \citep{BREDERECK2016140}[Sec. 3.3]. The cost $c_D$ arises when $A_1$ is shifted upward in the ranking of $D$, while paying nothing leaves the ranking of $D$ unchanged.

A shift bribery action under all-or-nothing price function is specified by a set of voters to bribe. For each bribed voter, the preferred candidate $p$ may be shifted upward by any positive number of positions at fixed cost $c_D$, while the relative order of all other candidates remains the same.

It is enough to observe that every feasible subset $S \subseteq \mathcal{D}$ induces the same post-training profile as bribing exactly the voters corresponding to datasets in $S$ and shifting $A_1$ to the top of their rankings.

First, suppose that $S \subseteq \mathcal{D}$ is a feasible subset of benchmark datasets for BST, i.e., it holds that $\sum_{D \in S}c_D \leq \beta$ and training on $S$ makes $A_1$ the top element under $\tilde{B}$. As argued in Section~\ref{BSTform}, for ordinal monotone operators we
may assume without loss of generality that the model developer uses maximal gains
$g_D=G_D$ on every chosen dataset $D\in S$. By the assumption that
$G_D$ is large enough to move $A_1$ to the top of that dataset
ranking, training on $S$ makes $A_1$ top-ranked in $\succ_D^T$ for each $D \in S$. For each dataset $D \notin S$, the dataset rankings are unchanged, so $\succ^{T}_D=\succ^{0}_D$. Moreover, on every dataset in $\mathcal{D}$, the relative order of all models in $\mathcal{A} \setminus \{A_1\}$ remains unchanged. Therefore, bribing exactly the voters corresponding to datasets in $S$ and shifting $A_1$ to the top of their rankings produces the same profile of dataset rankings $R^T$. The cost of this shift bribery is $\sum_{D \in S}c_D$, which is at most $\beta$ because $S$ is feasible for BST. Furthermore, since the resulting profile is the same and $E=\tilde{B}$, the preferred candidate $p=A_1$ is a winner/top element under the election rule. It follows that the same choice of voters gives a feasible solution to the corresponding shift bribery problem.

Conversely, suppose there is a feasible all-or-nothing shift bribery action of cost at most $\beta$ that makes $p=A_1$ a winner under $E=\tilde{B}$. For each voter corresponding to a dataset $D$, any nonzero shift of $A_1$ costs $c_D$, regardless of how far $A_1$ is shifted. Since $\tilde{B}$ is monotone, moving $A_1$ further upward in any individual ranking cannot lower its position in the aggregate ranking. Hence, we may assume without loss of generality that every bribed voter shifts $A_1$ to the top of their ranking. Let $S$ be the set of datasets whose corresponding voters are bribed in this shift bribery action. Since the action has cost at most $\beta$, we have $\sum_{D \in S}c_D \leq \beta$. Now suppose we train on exactly the datasets in $S$ and set maximal gains $g_D=G_D$. This makes $A_1$ the top-ranked model on every $D \in S$. For every $D \notin S$, the rankings remain unchanged and for every dataset, the relative order of all other models is unchanged. Thus, training on $S$ yields the same profile of dataset rankings as the successful shift bribery action. Since that profile makes $A_1$ the top element under $\tilde{B}$, the subset $S$ is feasible for BST. Therefore, feasible subsets for BST and feasible all-or-nothing shift bribery actions correspond to one another under the stated identification.
\end{proof}

\begin{corollary}\label{bordaapp}
The BST problem is NP-hard under $B_\text{Borda}$.
\end{corollary}
\begin{proof}
Consider an instance of shift bribery with all-or-nothing prices under Borda, consisting of voters with strict preference orders over $n$ candidates, a preferred candidate, a positive cost for each voter and a budget. Now we construct a BST instance under the identification of Proposition~\ref{iso}: datasets are voters, models are candidates and the target model $A_1$ is the preferred candidate. On each dataset, we assign metric values that preserve the corresponding voter's preference order: the top-ranked model receives $(n-1)/n$, the second-ranked model receives $(n-2)/n$ and so on down to $0$ for the lowest-ranked model. Before training, the profile induced by these metric values consists exactly of the original preference orders; moreover, all values within a dataset are distinct so that each dataset induces a strict linear order. The costs and the budget are those of the shift bribery instance. For every dataset $D$, we set the maximal gain to $G_D=1-\phi^0(A_1,D)$. Training on $D$ therefore sets the target model's post-training score to $1$. Since every other model has score at most $(n-1)/n$, $A_1$ is strictly top-ranked on $D$, so the assumption that the maximal gains are large enough holds by construction. The construction assigns one metric value to every model on every dataset, and all metric values and maximal gains have denominator $n$. Thus, it can be constructed in polynomial time.

We now make the two directions from Proposition~\ref{isoapp} explicit for this pair of instances. First, suppose that training on a feasible subset makes $A_1$ a top element. Any chosen dataset on which $A_1$ is already top-ranked can be omitted without changing the profile. Bribing the voters corresponding to the remaining datasets and shifting the preferred candidate to the top then produces the same profile and costs at most the budget. Conversely, assume there is a feasible all-or-nothing shift bribery action of
cost at most the budget that makes the preferred candidate a winner of the election. Since Borda count is monotone and the cost does not depend on the shift amount, we may assume that the preferred candidate is shifted to the top of every bribed voter's ranking. Training on the corresponding datasets then produces the same profile and costs at most the budget. Thus, the answer to the constructed BST problem is yes if and only if the answer to the given shift bribery problem is yes.

The hardness construction in \citet[Thm.~7]{Elkind2009} uses positive costs. Its price functions are all-or-nothing, as observed in \citet[Sec.~3.3]{BREDERECK2016140}. Therefore, the BST problem is NP-hard under $B_{\text{Borda}}$. As stated in Section~\ref{bribe}, under strict preference relations Borda count and mean win rate induce the same overall ranking, so the same conclusion holds for mean win rate.
\end{proof}

\paragraph{Algorithmic implications of the identification.}
Under Borda count with uniform all-or-nothing prices, \citet[Thm.~8]{FALISZEWSKI2021} give a polynomial-time algorithm that affects at most $1.5$ times the minimum number of voters plus one. Under our identification, this gives the same guarantee for the minimum number of datasets under Borda and, under strict preference relations, mean win rate.

More generally, \citet[Thm.~12]{FALISZEWSKI2021} give polynomial-time approximations for positional scoring rules with cost arbitrarily close to the minimum. This result applies to all-or-nothing prices even when costs differ across datasets. For Borda count, \citet[Prop.~2]{BREDERECK2016140} also give an exact fixed-parameter algorithm under all-or-nothing prices parameterized by the total number of voters, which corresponds to the total number of datasets in our setting. Under strict preference relations, these algorithmic conclusions also hold for mean win rate.

For exact computation, shift bribery under Borda remains W[2]-hard for all-or-nothing prices when parameterized by the number of affected voters \citep[Thm.~5]{BREDERECK2016140}. Under our identification, the affected voters correspond to the datasets included in benchmark-specific training. Thus, an exact fixed-parameter algorithm in the number of selected datasets is unlikely for BST under Borda count. Under strict preference relations, this holds for mean win rate.

\begin{theorem}\label{meanapp}
$k_\text{mean}(\phi^0, A_1)= \text{min} \Bigl\{k \in \{0, \dots, m \}: \sum_{i=1}^k G_{(i)} \geq m \cdot \Delta_\text{mean} (\phi^0, A_1)\Bigl\}$ with $k_\text{mean}(\phi^0, A_1)=+ \infty$ if no such $k$ exists.
\end{theorem} 

\begin{proof}
 Since benchmark-specific training only affects the metric values of model $A_1$, we have $$\mu_{\phi^T}(A_1)=\mu_{\phi^0}(A_1)+ \frac{1}{m}\sum_{D \in \mathcal{D}} g_D, \ \ \ \mu_{\phi^T}(A)=\mu_{\phi^0}(A)\ \forall A \neq A_1 .$$ Thus, $A_1$ is the top element of $B_\text{mean}(\phi^T)$ iff $$\sum_{D \in \mathcal{D}} g_D \geq m \cdot \Delta_\text{mean}(\phi^0, A_1).$$ Among all subsets of datasets of size $k$, the maximal total improvement is $ \sum_{i=1}^k G_{(i)}$ which is obtained by selecting the $k$ datasets with largest gains. Therefore, the smallest possible $k$ is exactly $\text{min} \Bigl\{k \in \{0, \dots, m \}: \sum_{i=1}^k G_{(i)} \geq m \cdot \Delta_\text{mean} (\phi^0, A_1)\Bigl\}$. If no such $k\leq m$ exists, then $k_\text{mean}(\phi^0, A_1)=+ \infty$.
\end{proof}

\begin{theorem}\label{medianapp}
$k_\text{median}(\phi^0, A_1)=\Delta_{\text{median}}(\phi^0, A_1) \ \text{if} \ C_{\tau(\phi^0)}(\phi^0, A_1)\geq \Delta_{\text{median}}(\phi^0, A_1).$ If $C_{\tau(\phi^0)}(\phi^0, A_1)<  \Delta_{\text{median}}(\phi^0, A_1)$, then $k_\text{median}(\phi^0, A_1)=+\infty$.
\end{theorem}
\begin{proof}
The condition $\tilde{\phi}^T (A_1) \geq \tau (\phi^0)$ holds if and only if at least $m-h+1$ datasets $D \in \mathcal{D}$ satisfy $\phi^T (A_1,D) \geq \tau (\phi^0)$. After sorting the $m$ metric values, it holds that $\tilde{\phi}^T (A_1)=\phi^T_{(h)}(A_1)$, so the median reaches the threshold $\tau (\phi^0)$ exactly when the values in positions $h, h+1, \dots, m$ all reach the threshold and thus there are $m-h+1$ such values.

By definition, there are already $N_{\tau(\phi^0)}(\phi^0, A_1)$ datasets on which $A_1$ reaches $\tau (\phi^0)$ before training and these datasets continue
to reach it after training. Among the remaining datasets, exactly $C_{\tau(\phi^0)}(\phi^0, A_1)$ datasets can be pushed to at least $\tau (\phi^0)$ by benchmark-specific training. All other datasets remain below $\tau (\phi^0)$ even under maximal gain. Therefore, the only useful datasets for increasing the median are the $C_{\tau(\phi^0)}(\phi^0, A_1)$ datasets that can be pushed across the
threshold and each of those datasets can contribute exactly one new metric value that reaches $\tau (\phi^0)$. Then, the number of additional datasets needed on which $A_1$ must reach $\tau (\phi^0)$ is $$\Delta_{\text{median}}(\phi^0, A_1):=\text{max} \{0, m-h+1-N_{\tau(\phi^0)}(\phi^0, A_1) \}.$$ If $C_{\tau(\phi^0)}(\phi^0, A_1)\geq \Delta_{\text{median}}(\phi^0, A_1)$, then choosing any $\Delta_{\text{median}}(\phi^0, A_1)$ of these datasets suffices. No smaller subset can work, since each chosen dataset can add at most one new value at or above the threshold. Hence, $$k_\text{median}(\phi^0, A_1)=\Delta_{\text{median}}(\phi^0, A_1).$$ If $C_{\tau(\phi^0)}(\phi^0, A_1)< \Delta_{\text{median}}(\phi^0, A_1)$, then even training on all datasets that can potentially be pushed across $\tau (\phi^0)$ does not suffice for $\tilde{\phi}^T (A_1) \geq \tau (\phi^0)$ to hold. Thus, no subset $S$ is sufficient and $k_\text{median}(\phi^0, A_1)=+\infty$.
\end{proof}

\begin{lemma}\label{windiffapp}
After benchmark-specific training on $S$ with $g_D=G_D$ for all $D \in S$ it holds that $$w(\phi^T, A) -w(\phi^T, A_1)=d_\text{win}(\phi^0, A_1,A) - \frac{1}{m}\sum_{D \in S}q_D(A) \ \ \text{for all}  \ A \neq A_1.$$    
\end{lemma}
\begin{proof}
By definition of the benchmark specific training, $\phi^{0}(A, D)=\phi^{T}(A,D)$ for all $A\neq A_1$ and all $D \in \mathcal{D}$. For $D \notin S$, $\phi^{0}(A_1, D)=\phi^{T} (A_1,D)$, so all win rates for these datasets remain unchanged. For $D \in S$, the benchmark specific training changes the difference in win rate of $A$ and $A_1$ as compared to the default evaluation protocol by $$(w_D(\phi^T, A)-w_D(\phi^0, A)) - (w_D(\phi^T, A_1)-w_D(\phi^0,A_1))=-q_D(A).$$ Averaging over all datasets $D\in \mathcal{D}$ yields $$(w(\phi^T, A)-w(\phi^T, A_1))=(w(\phi^0, A) -w(\phi^0, A_1))-\frac{1}{m}\sum_{D \in S}q_D(A).$$ Substituting the definition of $d_\text{win}$ gives the result.
\end{proof}

\begin{theorem}\label{pairwisewinapp}
$k_\text{win}(\phi^0, A_1,A)=\text{min}\Bigl\{k \in \{0, \dots, m \}: \sum_{i=1}^k q_{(i)}(A) \geq m \cdot \Delta_{\text{win}}(\phi^0, A_1,A) \Bigl\}$ with $k_\text{win}(\phi^0, A_1, A)=+ \infty$ if no such $k$ exists.
\end{theorem}
\begin{proof}
By Lemma~\ref{windiffapp}, $A_1$ is ranked at least as high as $A$ under $B_\text{win}(\phi^T)$ if and only if $\sum_{D \in S}q_D(A)\geq m \cdot d_\text{win}(\phi^0, A_1,A).$ If $d_\text{win}(\phi^0, A_1,A)<0$, then the above inequality holds for $S=\emptyset$ since all $q_D(A)\geq 0$. Hence, $k_\text{win}(\phi^0, A_1,A)=0$. Otherwise $\Delta_{\text{win}}(\phi^0, A_1,A)=d_{\text{win}}(\phi^0, A_1,A)>0$ and since $q_D(A)\geq 0$ for every $D$, the maximum of $\sum_{D \in S} q_{D}(A)$ over subsets $S$ of size $k$ is $ \sum_{i=1}^k q_{(i)}(A)$, which is achieved by selecting the $k$ datasets with largest pairwise gains. Thus, the smallest $k$ for which some $S$ with $|S|=k$ satisfies $\sum_{D \in S}q_D(A)\geq m \cdot d_\text{win}(\phi^0, A_1,A)$ is exactly $\text{min}\Bigl\{k \in \{0, \dots, m \}: \sum_{i=1}^k q_{(i)}(A) \geq m \cdot \Delta_{\text{win}}(\phi^0, A_1,A) \Bigl\}$. If no $k\leq m$ achieves this, then no $S \subseteq\mathcal{D}$ does either and $k_\text{win}(\phi^0, A_1, A)=+ \infty$.
\end{proof}

\begin{theorem}\label{globalwinapp}
$k_\text{win}(\phi^0, A_1)=\text{min}\sum_{D\in\mathcal{D}}x_D \ \ \text{s.t.} \ \ \sum_{D \in \mathcal{D}}q_D(A) x_D\geq m \cdot \Delta_\text{win}(\phi^0, A_1,A) \ \forall A \in \mathcal{A}\setminus \{ A_1\}$  with $k_\text{win}(\phi^0, A_1)=+ \infty$ if the program is infeasible.
\end{theorem}
\begin{proof}
By Lemma~\ref{windiffapp}, $A_1$ is ranked at least as high as $A$ under $B_\text{win}(\phi^T)$ if and only if $\sum_{D \in S}q_D(A)\geq m \cdot \Delta_\text{win}(\phi^0, A_1,A)$; when  $d_\text{win}(\phi^0, A_1,A)\leq 0$, then the inequality is automatic because $q_D(A)\geq 0$, so we can replace $d_\text{win}(\phi^0, A_1,A)$ by $\Delta_\text{win}(\phi^0, A_1,A)$. Hence, $k_\text{win}(\phi^0, A_1)$ is the minimum cardinality of the subset $S \subseteq \mathcal{D}$ such that $\sum_{D \in S}q_D(A)\geq m \cdot \Delta_\text{win}(\phi^0, A_1,A)$ holds for all $A \in \mathcal{A}\setminus \{A_1\}$. Introducing $x_D=1$ if $D\in S$ and $x_D=0$ if $D \notin S$ yields the stated program. If there is no $S\subseteq \mathcal{D}$ such that all $n-1$ constraints are satisfied simultaneously, then $k_\text{win}(\phi^0, A_1)=+ \infty$.
\end{proof}

\begin{proposition}\label{lowerboundwinapp}
$k_{\text{win}}(\phi^0, A_1) \geq \text{max}_{A \in \mathcal{A}\setminus \{A_1\}} k_{\text{win}}(\phi^0, A_1, A)$.  
\end{proposition}
\begin{proof}
Let $x$ be any feasible solution for the integer problem in  Theorem~\ref{globalwinapp} and let $S\subseteq\mathcal{D}$ be the set of datasets selected by $x$. Then, $|S|=\sum_{D\in\mathcal{D}}x_D$ and, for every $A \in \mathcal{A}\setminus \{A_1\}$, feasibility of $x$ implies $\sum_{D \in S}q_D(A)\geq m \cdot \Delta_\text{win}(\phi^0, A_1,A)$, since these are the constraints of the program. Hence, $S$ is feasible for each pairwise problem. By Theorem~\ref{pairwisewinapp}, for every fixed $A \in \mathcal{A}\setminus \{A_1\}$ any $S$ must satisfy $|S|\geq k_{\text{win}}(\phi^0, A_1, A)$. Thus, $|S|\geq \text{max}_{A \in \mathcal{A}\setminus \{A_1\}}k_{\text{win}}(\phi^0, A_1, A)$. Since this holds for every feasible $S$, taking the minimum over all such $S$ (by Definition~\ref{robustdef}) yields $k_{\text{win}}(\phi^0, A_1) \geq \text{max}_{A \in \mathcal{A}\setminus \{A_1\}} k_{\text{win}}(\phi^0, A_1, A)$.
\end{proof}

\begin{theorem}\label{pairwisemajapp}
$k_{\text{maj}}(\phi^0,A_1,A)=\Delta_{\text{maj}}(\phi^0,A_1,A)$.    
\end{theorem}

\begin{proof}
For every $D\in S$, the indicator $\mathbf{1}\{\phi^T(A_1,D)\geq\phi^T(A,D)\}$ is non-decreasing in $g_D$, this implies that we can set $g_D=G_D$ on every $D\in S$ without loss of generality. This makes $A_1$ top-ranked in $\succ_D^T$, for each $D\in S$. Hence, this indicator equals $1$ for every $D\in S$, equals $0$ for $D\in L^0(A)\setminus S$ and equals $1$ for every $D \notin L^0(A)\cup S$. Therefore, $$M(\phi^T,A_1,A)=M(\phi^0,A_1,A) + |S\cap L^0(A)| = m - |L^0(A)| + |S\cap L^0(A)|.$$ Then, the condition $M(\phi^T,A_1,A)\geq \mu$ becomes $|S\cap L^0(A)|\geq \Delta_{\text{maj}}(\phi^0,A_1,A)$. Since datasets outside $L^0(A)$ cannot increase the pairwise majority count against the competing model $A$, the minimum is attained by choosing any $S\subseteq L^0(A)$ of size $\Delta_{\text{maj}}(\phi^0,A_1,A)$. Such a set exists because $\mu\leq m$ implies $\Delta_{\text{maj}}(\phi^0,A_1,A)\leq |L^0(A)|$.
\end{proof}

\begin{theorem}\label{globalmajapp}
$k_{\text{maj}}(\phi^0,A_1)=\text{min}\sum_{D\in\mathcal{D}}x_D \ \ \text{s.t.} \ \ \sum_{D \in L^0(A)}x_D \geq \Delta_\text{maj}(\phi^0, A_1,A) \ \forall A \in \mathcal{A}\setminus \{ A_1\}$    
\end{theorem}

\begin{proof}
By Theorem~\ref{pairwisemajapp}, $M(\phi^T,A_1,A)\geq \mu$ holds if and only if $|S\cap L^0(A)| \geq \Delta_{\text{maj}}(\phi^0,A_1,A)$. Thus, $k_{\text{maj}}(\phi^0,A_1)$ is the minimum cardinality of a subset $S \subseteq \mathcal{D}$ for which this inequality holds for every $A \in \mathcal{A} \setminus \{A_1\}$ simultaneously. Setting $x_D=1$ if $D \in S$ and $x_D=0$ otherwise yields the stated program. Setting $x_D=1$ for every $D \in \mathcal{D}$ satisfies every constraint, so the program is always feasible.    
\end{proof}

\begin{proposition}\label{lowerboundmajapp}
$k_{\text{maj}}(\phi^0, A_1) \geq \text{max}_{A \in \mathcal{A}\setminus \{A_1\}} \Delta_\text{maj}(\phi^0, A_1, A)$.  
\end{proposition}

\begin{proof}
Let $x$ be any feasible solution for the integer problem in  Theorem~\ref{globalmajapp} and let $S\subseteq\mathcal{D}$ be the corresponding subset. For every $A \in \mathcal{A} \setminus \{A_1\}$, feasibility of $x$ implies $|S \cap L^0(A)|\geq \Delta_{\text{maj}}(\phi^0,A_1,A)$. Hence, $|S| \geq |S \cap L^0(A)| \geq \Delta_{\text{maj}}(\phi^0,A_1,A)=k_{\text{maj}}(\phi^0,A_1,A)$ by Theorem~\ref{pairwisemajapp}. Taking the maximum over $A$ and the minimum over feasible $S$ yields the lower bound. 
\end{proof}

\paragraph{Heterogeneous costs.}
If we keep the different costs from Section~\ref{framework} and replace $\min\sum_{D\in\mathcal D}x_D$ by $\min\sum_{D\in\mathcal D}c_Dx_D$ in Theorems~\ref{globalwin} and~\ref{globalmaj}, both integer programs remain otherwise unchanged. Under median, the solution selects the $\Delta_{\text{median}}(\phi^0,A_1)$
cheapest datasets that can be pushed across $\tau(\phi^0)$. If there are not enough such datasets, no subset can make the target model top-ranked. Under arithmetic mean, the sorting argument becomes a binary covering program.

\subsection{Experiments}\label{appexp}

\paragraph{Experimental setup.}
For HELM, multiple run entries with different evaluation settings may appear for the same model-subject pair; in that case, we keep the most complete one and if several runs are tied, we select one deterministically. Then, we average within the retained entry to obtain one metric value for each model-subject pair. For BBH, if multiple entries exist for the same model-task pair, we average them to a single value. In both benchmark suites, we only consider the models that have scores on every task and drop any datasets with missing values. This yields a complete $m \times n$ matrix with tasks as rows and models as columns under default protocol $P_0$.

\paragraph{Computation and uncertainty.} For arithmetic mean and median, robustness is computed directly from the expressions in Theorems~\ref{mean} and~\ref{median}. For mean win rate and pairwise majority, all reported values are the global robustness values: for each target model we solve the integer programs in Theorems~\ref{globalwin} and~\ref{globalmaj}. Thus, each value $k_B(\phi^0,A_1)$ is deterministic, given the fixed score matrix, the model set and the aggregation rule. Uncertainty arises only when we summarize these values across target models, for instance when we give the fraction of models that can be made top-ranked after training on at most $K$ datasets. Such summaries would be too confident if every model entry were treated as an independent observation, as many entries are closely related variants from the same developer or the same Hugging Face namespace. This is why we use a bootstrap over namespaces. A namespace is the lower-case part of the model identifier before the first slash. In each bootstrap sample, we resample namespaces with replacement and include all model entries belonging to each sampled namespace. This keeps related model entries together and measures how stable the summary is to the particular set of namespaces observed in the benchmark. We use $B=10000$ bootstrap resamples; the resulting summaries are in Table~\ref{tab:robust_sum}.

\begin{table}[h]
\centering
\caption{Robustness summaries for MMLU and full BBH. Entries under $K\leq 5$ and $K\leq 10$ give the percentage of target models that can be made top-ranked by training on at most $K$ datasets, with 95\% bootstrap intervals in brackets. The final column gives the median $k$ for each rule, with the same bootstrap interval. For example, $13.64$ [$0.00$, $31.58$] means that $13.64\%$ of target models in HELM MMLU can be made top-ranked under arithmetic mean by training on at most $5$ of the $57$ subjects. The interval says that, after resampling model families, values between $0.00\%$ and $31.58\%$ remain compatible with the observed variation across models.}
\label{tab:robust_sum}
\begin{tabular}{llccc}
\toprule
Suite & Rule & $K\leq 5$ & $K\leq 10$ & Median $k$ \\
\midrule
MMLU & Arithmetic mean & $13.64$ [$0.00$, $31.58$] & $31.82$ [$15.00$, $45.83$] & $16.0$ [$13.0$, $22.0$] \\
MMLU & Median & $9.09$ [$0.00$, $18.18$] & $13.64$ [$0.00$, $31.58$] & $23.0$ [$19.0$, $28.0$] \\
MMLU & Mean win rate & $4.55$ [$0.00$, $11.54$] & $9.09$ [$0.00$, $18.18$] & $44.5$ [$41.0$, $50.0$] \\
MMLU & Pairwise majority & $4.55$ [$0.00$, $11.54$] & $9.09$ [$0.00$, $18.18$] & $29.0$ [$28.0$, $29.0$] \\
\midrule
BBH & Arithmetic mean & $0.98$ [$0.56$, $1.50$] & $19.30$ [$14.62$, $24.48$] & $13.0$ [$13.0$, $13.0$] \\
BBH & Median & $1.04$ [$0.57$, $1.65$] & $15.46$ [$11.52$, $20.05$] & $12.0$ [$12.0$, $12.0$] \\
BBH & Mean win rate & $0.49$ [$0.23$, $0.81$] & $3.35$ [$1.96$, $5.34$] & $22.0$ [$22.0$, $22.0$] \\
BBH & Pairwise majority & $0.33$ [$0.12$, $0.60$] & $6.17$ [$3.97$, $9.28$] & $12.0$ [$12.0$, $12.0$] \\
\bottomrule
\end{tabular}
\end{table}

\paragraph{Normalized robustness.} We also examine normalized robustness values, which adjust the raw number of required datasets by a rule-specific reference value capturing how much useful improvement is available to the target model. This addresses the concern that raw robustness may mix two effects: the vulnerability of the aggregation rule and the amount of room the target model has left to improve. The exact denominators are as follows. For arithmetic mean, for target models with positive mean deficit, we divide by $\text{min}\{\lceil m\Delta_\text{mean}(\phi^0,A_1)/(1-\mu_{\phi^0}(A_1))\rceil,m\}$, where $\mu_{\phi^0}(A_1)$ is the target model's mean score under the default evaluation protocol. For median, we divide by $C_{\tau(\phi^0)}(\phi^0,A_1)$, the number of datasets that can be pushed across the threshold $\tau(\phi^0)$. For mean win rate, we take a competing model $A\neq A_1$ attaining $\text{max}_{A'\neq A_1}\Delta_\text{win}(\phi^0,A_1,A')$ among competitors with positive pairwise win-rate deficit and divide by $\text{min}\{\lceil m\Delta_\text{win}(\phi^0,A_1,A)/(m^{-1}\sum_{D\in\mathcal D}q_D(A))\rceil,m\}$.  For pairwise majority, we divide by $|\{D\in\mathcal D: \exists A\in\mathcal A\setminus\{A_1\}, \Delta_\text{maj}(\phi^0,A_1,A)>0 \text{ and } D\in L^0(A)\}|$, the number of datasets on which the target model loses to at least one competitor with positive pairwise majority deficit. When the raw robustness value is zero, the normalized value is set to zero.

The results in Table~\ref{tab:norm} show that mean win rate has the largest median normalized robustness in both suites. Nonetheless, we detect a difference in the interpretation of the robustness for arithmetic mean. While its raw robustness is low in terms of the fraction of benchmark tasks that must be manipulated, its median normalized robustness is relatively high, especially on BBH. This means that its low raw robustness, discussed in Section~\ref{exp}, can be partly explained by the fact that the target model has large useful gains on some datasets; in other words, a few tasks with high gains can move the arithmetic mean substantially.  

\begin{table}[h]
\centering
\caption{Median normalized robustness values. Entries give the median of the normalized robustness values across target models. Normalized robustness divides $k_B(\phi^0,A_1)$ by a rule-specific reference value capturing the amount of useful improvement available to the target model. The second column contains values for MMLU, the third for BBH with all model entries and the fourth for BBH restricted to one model per namespace.}
\label{tab:norm}
\begin{tabular}{lccc}
\toprule
Rule & MMLU & Full BBH & BBH per namespace \\
\midrule
Arithmetic mean & $0.57$ & $0.81$ & $0.81$ \\
Median & $0.45$ & $0.50$ & $0.50$ \\
Mean win rate & $0.90$ & $0.92$ & $0.91$ \\
Pairwise majority & $0.51$ & $0.50$ & $0.50$ \\
\bottomrule
\end{tabular}
\end{table}

\paragraph{Additional results and sensitivity analyses.} BBH contains many related uploads from the same Hugging Face namespace. To check that the robustness results are not driven by large namespaces with many model variants, we repeat the analysis after keeping only the model with the highest mean score within each namespace. This reduces the target set from $4507$ model entries to $718$ namespace entries. The conclusion is unchanged: mean win rate remains the hardest rule to manipulate, with median robustness $21$ of $24$ tasks, compared with $13$ under arithmetic mean and $12$ under median and pairwise majority (see Figure~\ref{dup_ecdf} and Table~\ref{tab:dup}). 

\begin{figure}[h]
\centering
\includegraphics[width=\linewidth]{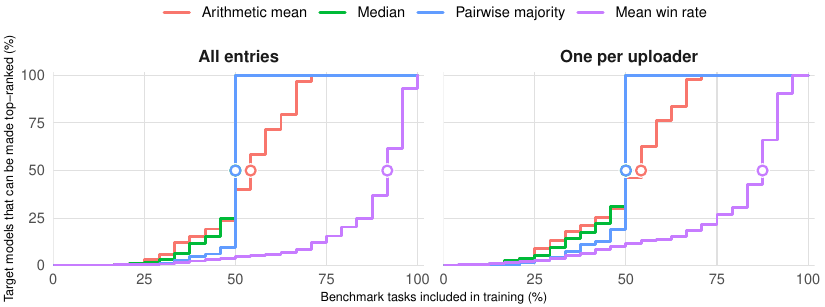}
\caption{BBH sensitivity analysis for related uploads from the same Hugging Face namespace. Treating each model in turn as the target, we compute the minimum fraction of BBH tasks on which benchmark-specific training would make the target top-ranked and plot the empirical CDF across targets. The left plot uses all $4507$ BBH model entries. The right plot keeps only the model with the highest mean score within each namespace, leaving $718$ entries. Curves farther to the right indicate greater robustness. Dots mark the median robustness for each rule.}
\label{dup_ecdf}
\end{figure}

\begin{table}[h]
\centering
\caption{BBH robustness analysis when retaining one model per Hugging Face namespace. The second column uses all $4507$ BBH model entries. The third column keeps only the model with the highest mean score in each namespace, resulting in $718$ entries. Values are median robustness values $k$ with the percentage of BBH tasks in parentheses.}
\label{tab:dup}
\begin{tabular}{lcc}
\toprule
Rule & Full BBH & One model per namespace \\
\midrule
Arithmetic mean & $13$ ($54.2\%$) & $13$ ($54.2\%$) \\
Median & $12$ ($50.0\%$) & $12$ ($50.0\%$) \\
Mean win rate & $22$ ($91.7\%$) & $21$ ($87.5\%$) \\
Pairwise majority & $12$ ($50.0\%$) & $12$ ($50.0\%$) \\
\bottomrule
\end{tabular}
\end{table}

Furthermore, because each target model is evaluated under all four aggregation rules, we use paired Wilcoxon signed-rank tests to compare rules within the same target models. The tests confirm that the differences between aggregation rules are systematic rather than driven by a few models; the results for mean win rate are in Table~\ref{tab:wilcoxon}.

\begin{table}[t]
\centering
\caption{How much more manipulation is required under mean win rate? Each entry is the median paired difference between robustness under mean
win rate and robustness under the given rule, measured in
percentage points of the benchmark and given in numbers of benchmark datasets in parentheses. Positive values mean that mean win rate is more robust and requires more datasets to be manipulated. The second column contains values for MMLU, the third for BBH with all $4507$ model entries and the fourth for BBH restricted to one model per namespace. The paired Wilcoxon signed-rank tests remain significant after Holm correction with $p_{\text{adj}}<0.001$.
}
\label{tab:wilcoxon}
\begin{tabular}{lccc}
\toprule
Baseline rule & MMLU & Full BBH & BBH per namespace \\
\midrule
Arithmetic mean
  & $44.7$ pp ($25.5$ datasets)
  & $33.3$ pp ($8$ datasets)
  & $29.2$ pp ($7$ datasets) \\
Median
  & $37.7$ pp ($21.5$ datasets)
  & $41.7$ pp ($10$ datasets)
  & $37.5$ pp ($9$ datasets) \\
Pairwise majority
  & $27.2$ pp ($15.5$ datasets)
  & $41.7$ pp ($10$ datasets)
  & $37.5$ pp ($9$ datasets) \\
\bottomrule
\end{tabular}
\end{table}

\begin{table}[b]
\centering
\caption{Number of target models for which each comparison holds, with the number for which the inequality is strict in parentheses. Values are raw robustness values.}
\label{tab:comparisons}
\begin{tabular}{lcc}
\toprule
Comparison & MMLU & Full BBH \\
\midrule
$k_{\text{win}}\geq k_{\text{mean}}$
& $22/22\;(21)$ & $4492/4507\;(4489)$ \\
$k_{\text{win}}\geq k_{\text{median}}$
& $22/22\;(21)$ & $4490/4507\;(4456)$ \\
$k_{\text{win}}\geq k_{\text{maj}}$
& $21/22\;(20)$ & $4473/4507\;(4425)$ \\
\bottomrule
\end{tabular}
\end{table}

We also investigate how often mean win rate is at least as robust as each other aggregation rule. Table~\ref{tab:comparisons} gives the number of target models for which each comparison holds. These comparisons hold for almost all target models in both suites, but do not establish a general ordering of the aggregation rules.

Moreover, we examine whether the normalized robustness is related to the baseline strength of the target model. For each aggregation rule, we compute Spearman's $\rho$ between a model's mean score under the default protocol and its normalized robustness. The Spearman rank correlation is negative under every rule. On MMLU, $\rho$ ranges from $-0.98$ under arithmetic mean, median and mean win rate to $-0.89$ under pairwise majority. In the BBH analysis with one model per namespace, $\rho$ is $-0.94$ under arithmetic mean and $-0.92$ under mean win rate, $-0.77$ under median and $-0.65$ under pairwise majority. The results are intuitive: models that already have high accuracy are often close to the top of the leaderboard, so benchmark-specific training may need to use only a small share of their available improvement opportunities to change their rank. Weaker models may have more room to improve, but they also start farther from the top and hence require a larger share of that room to be used. The manipulation risk is therefore concentrated among models that are already competitive on the benchmark, since improvements on relatively few or especially consequential tasks can be enough to affect the top of the leaderboard ranking. More robust aggregation rules such as mean win rate raise the cost of manipulation by requiring improvements across a broader set of task comparisons, but they do not by themselves remove the underlying rank incentive: when a model is already competitive, even a relatively small number of strategically chosen improvements may still be enough to change its leaderboard position. 

\paragraph{Models near the top of the leaderboard.}
 As an additional analysis, we examine the quarter of target models with the highest mean score under the default evaluation protocol. The competing models are unchanged within each analysis. Table~\ref{tab:top} shows the median raw and normalized robustness. Mean win rate has the largest median raw and normalized robustness in all three subsets.

\begin{table}[t]
\centering
\caption{Robustness for the quarter of target models with the highest mean score. Entries give median raw robustness, with median normalized robustness in parentheses. For BBH per namespace, robustness is computed on the leaderboard with $718$ models.}
\label{tab:top}
\begin{tabular}{llcccc}
\toprule
Suite & Targets
& $k_{\text{mean}}$ & $k_{\text{median}}$
& $k_{\text{win}}$ & $k_{\text{maj}}$ \\
\midrule
MMLU & $6/22$
& $6.5\;(0.451)$ & $9.5\;(0.247)$
& $21\;(0.590)$ & $18\;(0.356)$ \\
Full BBH & $1127/4507$
& $9\;(0.667)$ & $10\;(0.455)$
& $18\;(0.783)$ & $12\;(0.500)$ \\
BBH per namespace & $180/718$
& $7\;(0.583)$ & $8\;(0.400)$
& $13\;(0.684)$ & $10\;(0.435)$ \\
\bottomrule
\end{tabular}
\end{table}

\paragraph{Effect of model size.} 
Model size has an effect through the default metric values $\phi^0$, maximal gains $G_D$ and costs $c_D$. A developer who chooses the model size as well as the datasets solves our problem once for each available size and takes the minimum over sizes. We grouped all 4507 BBH model entries by the model size reported in the Open LLM Leaderboard evaluation results and recomputed the median robustness within each group (see Table~\ref{tab:size}).

\begin{table}[h]
\centering
\caption{Median robustness $k$ (out of $24$ BBH tasks) by model size.}
\label{tab:size}
\begin{tabular}{lcccc}
\toprule
Model size (entries) & Arithmetic mean & Median
& Mean win rate & Pairwise majority \\
\midrule
$\leq 3\mathrm{B}$ ($774$)
& $16$ & $12$ & $23$ & $12$ \\
$(3,8]\mathrm{B}$ ($1186$)
& $14$ & $12$ & $22$ & $12$ \\
$(8,15]\mathrm{B}$ ($2108$)
& $12$ & $12$ & $22$ & $12$ \\
$(15,40]\mathrm{B}$ ($243$)
& $10$ & $11$ & $19$ & $12$ \\
$>40\mathrm{B}$ ($196$)
& $8$ & $8$ & $14$ & $10$ \\
\bottomrule
\end{tabular}
\end{table}

Here, we observe two patterns. First, model size is positively correlated with a model’s mean score under the default protocol (Spearman’s $\rho$ = 0.71) and negatively correlated with robustness under every aggregation rule ($\rho$ from -0.71 to -0.41). This is consistent with the negative correlation between mean score and normalized robustness and with our observation that manipulation risk is concentrated among models that are already competitive. Second, mean win rate has the largest median robustness within every size group. The same conclusion holds for normalized robustness and after retaining one model with the highest mean task score within each namespace.

\paragraph{Heterogeneous costs.}
Moreover, we repeat the analysis with $c_D$ equal to the number of test instances in each dataset. MMLU subject sizes range from 100 to 1534 test instances and sum to 14,042. The 24 BBH task sizes range from 146 to 250 test instances and sum to 5,761. Under the same post-training scores as in the main analysis, the median minimum costs are given in Table~\ref{tab:costs}.

\begin{table}[h]
\centering
\caption{Median minimum costs when $c_D$ equals the number of test instances. Percentages in parentheses are shares of $\sum_{D\in\mathcal D}c_D$. For BBH, the analysis retains one model with the highest mean task score within each namespace.}
\label{tab:costs}
\begin{tabular}{lcccc}
\toprule
Suite & Arithmetic mean & Median & Mean win rate & Pairwise majority \\
\midrule
MMLU, $22$ targets & $2307\;(16.4\%)$ & $2652.5\;(18.9\%)$ & $8446\;(60.1\%)$ & $3592\;(25.6\%)$ \\
BBH, $718$ namespaces & $3011\;(52.3\%)$ & $2761\;(47.9\%)$ & $5011\;(87.0\%)$ & $2761\;(47.9\%)$ \\
\bottomrule
\end{tabular}
\end{table}

Mean win rate has the largest median minimum cost under this choice. Note that the number of test instances is only one possible choice of $c_D$ and does not measure the cost of including a dataset in training. Thus, the comparison need not persist for every choice of costs. 

\paragraph{Compute resources and software.}
All experiments were run on a local MacBook Pro (macOS 15.7.3) with an Apple M3 chip, 8 CPU cores and 24 GB memory, using CPU computation only and no GPU or cloud resources. The full robustness computation, bootstrap and statistical tests took approximately 1.5--2 hours. The experiments were run with R 4.4.2 and the R packages loaded in the released code.

\paragraph{Existing assets and licenses.}
We use only existing public benchmark scores. MMLU is credited to \cite{Hendrycks2021} and is distributed under the MIT License; HELM is credited through \cite{liang2023} and is distributed under Apache-2.0. BBH is credited to \cite{suzgun2022challengingbigbenchtaskschainofthought} and its public BIG-Bench-Hard repository is distributed under MIT; the BBH leaderboard scores are credited to the Hugging Face Open LLM Leaderboard and were produced with the EleutherAI Evaluation Harness \citep{eval-harness}, which is distributed under MIT. The Open LLM Leaderboard results are public Hugging Face Hub datasets, so we use them under the Hugging Face Hub Terms of Service and use only aggregate public leaderboard scores. The R code uses \texttt{jsonlite}, \texttt{dplyr}, \texttt{tidyr} and \texttt{tibble} under MIT, \texttt{digest} under GPL (>=2), and \texttt{lpSolve} under LGPL-2.

\subsection{Broader Impact}\label{impact}
This work is intended to improve the reliability and transparency of model evaluation by giving benchmark designers an exact way to measure how robust a fixed leaderboard is to benchmark-specific training. More robust aggregation can make leaderboard rank a less misleading signal of model capability and can support safer decisions about model comparison and deployment. A possible negative impact is that the same analysis could be read by model developers as information about how much benchmark-specific training is needed to rig a leaderboard. We mitigate this by focusing on aggregate robustness rather than task-level instructions for a particular target model.


\end{document}